\documentclass[11pt]{article}

\usepackage{fullpage}

\usepackage[utf8]{inputenc} % allow utf-8 input
\usepackage[T1]{fontenc}    % use 8-bit T1 fonts
\usepackage{hyperref}       % hyperlinks
\usepackage{url}            % simple URL typesetting
\usepackage{booktabs}       % professional-quality tables
\usepackage{amsfonts}       % blackboard math symbols
\usepackage{nicefrac}       % compact symbols for 1/2, etc.
\usepackage{microtype}      % microtypography
\usepackage{xcolor}         % colors

\usepackage[ruled,noend]{algorithm2e} % For algorithms

\usepackage{amsmath}
\usepackage{enumitem}
\usepackage{graphicx}
\usepackage{wasysym}
\usepackage{amsthm}

\newtheorem{theorem}{Theorem}[section]
\newtheorem{lemma}[theorem]{Lemma}
\newtheorem{proposition}[theorem]{Proposition}

\newcommand{\reals}{\mathbb{R}}

\renewcommand{\paragraph}[1]{
     \medskip\noindent\textbf{#1.} 
 }

\title{Online Algorithms with Limited Data Retention\thanks{A short version of this paper appeared in the 5th Annual Symposium on Foundations of Responsible Computing (FORC) in 2024.  The authors thank the anonymous referees of FORC, Rad Niazadeh, Stefan Bucher, the Simons Institute for the Theory of Computing and seminar participants at the 2024 SIGecom Winter Meetings and the 2022 C3.ai DTI Workshop on Data, Learning, and Markets.}}

\author{Nicole Immorlica\thanks{Microsoft Research, \texttt{nicimm@microsoft.com}} \and Brendan Lucier\thanks{Microsoft Research, \texttt{brlucier@microsoft.com}}
\and
Markus Mobius\thanks{Microsoft Research, \texttt{mobius@microsoft.com}} 
\and James Siderius\thanks{Tuck School of Business at Dartmouth, \texttt{james.siderius@tuck.dartmouth.edu}.  This research was initiated while the author was a Research Intern at Microsoft Research.}}
\date{}

\begin{document}

\maketitle

\begin{abstract}
We introduce a model of online algorithms subject to strict constraints on data retention. An online learning algorithm encounters a stream of data points, one per round, generated by some stationary process. Crucially, each data point can request that it be removed from memory $m$ rounds after it arrives. To model the impact of removal, we do not allow the algorithm to store any information or calculations between rounds other than a subset of the data points (subject to the retention constraints). At the conclusion of the stream, the algorithm answers a statistical query about the full dataset. We ask: what level of performance can be guaranteed as a function of $m$?

We illustrate this framework for multidimensional mean estimation and linear regression problems.  We show it is possible to obtain an exponential improvement over a baseline algorithm that retains all data as long as possible.  Specifically, we show that $m = \textsc{Poly}(d, \log(1/\epsilon))$ retention suffices to achieve mean squared error $\epsilon$ after observing $O(1/\epsilon)$ $d$-dimensional data points.  This matches the error bound of the optimal, yet infeasible, algorithm that retains all data forever.  We also show a nearly matching lower bound on the retention required to guarantee error $\epsilon$.  One implication of our results is that data retention laws are insufficient to guarantee the right to be forgotten even in a non-adversarial world in which firms merely strive to (approximately) optimize the performance of their algorithms.

Our approach makes use of recent developments in the multidimensional random subset sum problem to simulate the progression of stochastic gradient descent under a model of adversarial noise, which may be of independent interest.
\end{abstract}

\setcounter{page}{0}
\newpage

\section{Introduction}

Modern algorithms run on data.  But as the potential uses for large datasets have expanded, so too have concerns about personal data being collected, retained, and used in perpetuity. Data protection laws are one response to these concerns, providing individuals the right to have their data removed from datasets. The EU's GDPR is a flagship example, encoding a ``right to be forgotten'' and mandating compliance with data deletion requests~\cite{GDPR}.  Similar policies have taken effect across the United States, such as the California Consumer Privacy Act~\cite{CCPA} and Virginia's Consumer Data Protection Act~\cite{CDPA}. These policies specify rules governing deletion requests, but data removal is a complicated process.  Data is not just stored: it is used to make decisions; it touches a vast array of metrics; it trains machine learning models.  What, then, should it mean to remove data from a system?  And how do such requests impact an algorithm's ability to learn?

A growing body of literature approaches these questions through the ``outcome-based" lens of constraining the observed behavior and outcomes of an algorithm.  For example, one might require that once a piece of data has been ``removed'' in response to a request, the algorithm's behavior should be indistinguishable (in a cryptographic sense) from one that does not have access to the data.  Formalizing this idea leads to a myriad of details and modeling choices, and multiple notions of deletion-respecting algorithms have been proposed~\cite{cao2015towards,garg2020formalizing,cohen2023control}. An alternative approach is to directly impose restrictions on an algorithm's internal implementation that regulate and define the data removal process.  This ``prescriptive'' approach is especially appealing from a regulatory perspective, since such restrictions provide clear guidance on what is and is not allowed (and, by extension, what constitutes an enforceable violation).  But on the other hand, the actual implications of any given implementation restriction are not necessarily clear \emph{a priori}.  Constraints that appear very restrictive at first glance may still allow undesirable behavior through clever algorithm design.  This undesirable behavior may be exhibited even by non-adversarial firms that simply wish to optimize their performance.  Thus, for any given definition of what is meant by an implementation that respects data deletion, it is crucial to explore the  outcomes that are generated by optimal (or near-optimal) algorithm design.

In this paper we explore the latter approach.  We consider a stark framework in which an online algorithm can retain \emph{no state} beyond stored data, which is subject to deletion requests.  We show that even under such a restriction and even for simple statistical tasks like mean estimation, an algorithm that can preemptively delete data from its dataset can effectively retain information about data that was supposedly removed while still following the letter of the law (i.e., limited data retention).  Moreover, we show that one can use this flexibility to substantially improve performance on statistical tasks relative to naive baseline algorithms that follow the spirit of the law (i.e., the right to be forgotten).\footnote{For example, by retaining all data as long as is allowed by regulation and then using optimal statistical estimators on the retained dataset.}  These results suggest that even in a world where an algorithm can retain no internal state whatsoever beyond its dataset, the curation of the dataset itself can be used to encode substantial information about data that has been supposedly removed, and even non-adversarial designers who seek only to maximize performance may naturally develop algorithms that leak information that was requested to be deleted.  These results emphasize the importance of laws that regulate outcomes as well as process.

\paragraph{A Framework for Limited Data Retention}
We propose a framework for algorithm design built upon a literal interpretation of a request to remove data.  A sequence of data points is observed by a learning algorithm that actively maintains a subset of the data that has been observed so far.  Each data point can come with a request that it be stored for only $m$ rounds, after which it must be removed from the algorithm's subset.\footnote{All of our results extend directly to model where a removal request can be made in any round after the data arrives (not just at the moment of arrival), and the data must be removed within $m$ rounds of the request.} We think of $m$ as a legally-mandated period of time after which the algorithm is obligated to fulfill the request.\footnote{For example, under GDPR Article 12, any request to delete personal data must be honored ``Without undue delay and in any event within one month of receipt of the request''~\cite{GDPR}.}  The algorithm is free to discard data earlier, if desired; the only constraint is that data cannot be retained beyond the $m$ rounds.

Of course, removing data points from the ``official'' dataset has no bite without additional restrictions on what else the algorithm can store. To clarify the impact of removing data, we impose a crucial modeling assumption: the algorithm \emph{cannot retain any state} between rounds other than the dataset itself.  In other words, any statistics or intermediate calculations performed by the algorithm must be recomputed, when needed, using only the data currently in the dataset.\footnote{One can equivalently think of this as a policy describing which statistics can be kept between rounds; namely, those that could be directly recomputed using only the retained data.}

Such an algorithm can be described by two procedures: one that maintains the dataset (i.e., given the current subset and an incoming data point, choose which subset to keep) and one that answers a query about the full data stream given the current subset, possibly employing some non-standard estimator tailored to the data retention strategy.

We initiate an exploration of this framework through the lens of two standard statistical tasks: mean estimation and linear regression.  In the case of mean estimation, each data point is a drawn from an unknown distribution over $\reals^d$ and the algorithm's goal is to recover the distribution's mean.  In the case of linear regression, each data point is a pair $(x,y)$ where $x$ is a $d$-dimensional characteristic vector and $y$ is generated through a linear function of $x$ plus random noise, and the goal is to simulate the linear function on challenge queries.  In each case, the mean squared error achievable by an estimator that can retain an entire data stream of $T$ data points (without any requirement to remove data) improves linearly with $T$.  We ask: what error is achievable by an algorithm that respects requests to remove incoming data points within $m$ rounds?

One baseline algorithm is to simply retain all data as long as possible.  That is, the algorithm retains all of the previous $m$ data points, then returns the maximum likelihood estimator given the sample for the target query.  This approach is equivalent to keeping a uniform subsample of $m$ draws from the underlying distribution.  For the mean estimation and linear regression tasks, a uniform subsample of size $m$ yields an average squared error no better than $\Theta(1/m)$, even for draws from a Gaussian distribution.  In other words, this baseline would need to retain data for $m = O(T)$ rounds to achieve error comparable to what is attainable from the entire data stream.

\paragraph{An Improved Data Retention Policy}
We show that it is possible to achieve an exponential improvement relative to the baseline solution described above.  We present an algorithm for mean estimation that achieves a loss guarantee comparable to the optimal estimator over all $T$ data points, but that retains each data point for only $m = \textsc{Poly}(d, \log(T))$ rounds.  In more detail, if $m$ is at least $\Theta(d  \log(d/\epsilon))$, then for any query time $T > C d/\epsilon$ (where $C$ is a constant depending on the input distribution) the expected squared error will be at most $\epsilon$. For linear regression we achieve a similar guarantee, with $m = \Theta(d^2 \log(d) \log(d/\epsilon))$.  Our algorithms are polytime: each update step takes time linear in $d$ and $1/\epsilon$.

The idea behind these algorithms is to simulate the progression of stochastic gradient descent (SGD) on the incoming data stream.  As data arrives, an (unrestricted) learning algorithm could use SGD to adaptively maintain an estimate of the statistic of interest, but our framework rules out direct storage and maintenance of such an estimate. Instead, our algorithm adaptively curates a subset of recent data. Intuitively, since there are exponentially many possible subsets of recent data, there is very likely to exist \emph{some} subset $S$ for which the maximum likelihood estimator evaluated on $S$ is very close to any given target estimate value.  Recent developments in the multidimensional random subset sum problem~\cite{becchetti2022multidimensional} make this precise: very roughly speaking, the collection of all subsets of a pool of independent data points behaves similarly, in terms of the expected minimum distance of their averages to a target point, to exponentially many independent draws.  This drives the exponential improvement in our error rate, relative to preserving the full $m$ most recent data points.  In essence, the algorithm intentionally curates a subset of data that is intended to be ``more representative'' than independent samples.

We also present a nearly-matching lower bound: if $m = o\left( \frac{d \log(1/\epsilon)}{\log(d)\log\log(1/\epsilon)} \right)$ then the algorithm must have error greater than $\epsilon$ with constant probability, regardless of the output function used to map the final subsample to an estimate of the mean.  To provide some intuition for this result, note that an online algorithm must accomplish two tasks: it must learn the statistic of interest (say, the mean of the input distribution), and it must generate a subset of data from which that statistic can be recovered.  We show that even if we remove the learning component by endowing the algorithm with advance knowledge of (say) the distribution's mean, the lower bound would still apply to the task of constructing a subset that encodes it.  Indeed, since the algorithm has only $m$ samples from which to generate a final dataset, and hence only $2^m$ possibilities for the final subset itself, a union bound suggests that one cannot hope to obtain an average error rate better than exponentially small in $m$, regardless of the estimator used to convert datasets to estimates.

\paragraph{Roadmap}
In the remainder of this section we review additional related work.  In Section~\ref{sec:model} we describe our algorithmic framework and how it applies to the examples of mean estimation and linear regression.  In Section~\ref{sec:mean} we focus on mean estimation, starting with a simplified algorithm and then extending to an improved multi-dimensional estimation method in Section~\ref{sec:mean.alg.improved}.  Our lower bound for mean estimation appears in Section~\ref{sec:mean.lowerbound}.  Our analysis of the linear regression task appears in Section~\ref{sec:regression}.  We conclude and suggest future research directions in Section~\ref{sec:conclusion}.

\paragraph{Related Work}
There is a substantial line of literature that explores definitions of data removal, especially as it relates to data protection and privacy laws.  The literature on machine unlearning, initiated by~\cite{cao2015towards}, explores the process of updated a trained machine learning model so that it cannot leak information about to-be-deleted data.  This has led to a vast body of work exploring different definitions and designs; see~\cite{nguyen2022survey,xu2023machine} for some recent surveys of this literature.  Beyond machine learning contexts, a notion of data deletion in terms of not leaking information about the data and maintaining secrecy, termed deletion-as-confidentiality, was proposed by~\cite{garg2020formalizing}.  A more permissive notion that constrains the leakage of information only after a removal request, deletion-as-control, was explored by~\cite{cohen2023control}.  Such works employ outcome-based constraints on data leakage, often in combination with internal state restrictions. In contrast, we explore a prescriptive framework that directly restricts an algorithm's implementation and explore the extent to which these restrictions do (or do not) constrain the algorithm's achievable performance and observable outcomes. While our algorithm respects certain notions of random differential privacy~\cite{hall2011random}, we show that simple implementation restrictions to delete data points is not sufficient to retain full differential privacy of the deleted data (as we demonstrate in Appendix~\ref{app:df}).  

Our work is also related to a line of literature on non-uniform subsampling for linear regression.  The typical goal is to draw a sample from a large (or infinite) pool of potential data items $(x,y)$ to increase accuracy of resulting models. Early works employed leverage scores to weight the predictor vector $x$~\cite{drineas2012fast,ma2014statistical}. This approach has been extended to other norms via low-distortion embeddings~\cite{meng2013low} and improved by including outcomes $y$ via importance weighting~\cite{dhillon2013new,zhu2016gradient,ting2018optimal}. In contrast, our approach is not based on independent sampling but rather adaptive sample maintenance with elements added and removed over time.

Our approach is also closely related to coreset construction~\cite{feldman2020introduction}, in which the goal is to develop a highly compressed summary of a large dataset that retains the ability to answer queries from a given query class.  Effective constructions are known for many learning problems, including variations of regression for numerous risk functions~\cite{bachem2017practical}.  In principle a coreset can retain additional information beyond an (unweighted) subset of the original data, whereas our framework motivates us to focus specifically on unweighted subsampling.

From a technical perspective, our constructions use online implementations of stochastic gradient descent (SGD), which itself makes heavy use of sampling~\cite{bottou2010large}.  Our algorithms effectively simulate the progression of SGD using subsamples to approximate estimates. These approximations introduce some poorly-controlled noise to the SGD process, which necessitates an analysis that is robust to adversarial noise; for this we provide a slight variation on an SGD analysis due to~\cite{rakhlin2012making}.  To show that small subsets of data suffice to approximate the evolution of a sequence of improving estimates of regression coefficients, we employ recent advances in the theory of the random subset sum problem (RSS)~\cite{becchetti2022multidimensional,Lueker:1998,da2023revisiting}.  The application of the RSS problem in contexts where SGD is used has also been explored in literature related to the Strong Lottery Ticket Hypothesis (SLTH) in learning theory~\cite{ferbach2022general,becchetti2022multidimensional,pensia2020optimal}. However, the application of RSS to our setting requires a novel analysis.

\section{Model and Preliminaries}
\label{sec:model}

\paragraph{Online Algorithms for Statistical Estimation}
We consider a stream of data drawn from a stationary distribution.  There is a set $\mathcal{X}$ of potential data \emph{items} $x \in \mathcal{X}$.  The stream generates a (possibly infinite) sequence of items $x_1, x_2, \dotsc$.  We say that item $x_t$ arrives in round $t$.  An algorithm $A$ observes items $x_1, x_2, \dotsc$ in sequence and processes them online.  The algorithm maintains a \emph{state} that is updated in each round; we write $S_t$ for the state at the conclusion of round $t$.  The algorithm's update process can be viewed as a (possibly randomized) subroutine $A_{upd}$ that takes as input $t$, $x_t$, and $S_{t-1}$, and returns a new state $S_t$.\footnote{We will assume that each item is labeled with the round in which it arrives, so in particular the algorithm is aware of the current round.}  %Given an algorithm $A$ we will denote this subroutine by $A_{upd}(x_t, S_{t-1}, t)$.

We assume the data stream is generated by a stationary process, meaning that each item is drawn i.i.d.\ from a distribution $F$ over $\mathcal{X}$.  Distribution $F$ is assumed to lie in some known class $\mathcal{F}$ of potential distributions, but $F$ is not known to the algorithm or the designer.

After some (possibly unknown, possibly random) round $T \geq 1$ the algorithm will be asked a query $Q$ from a class of potential queries $\mathcal{Q}$.  Each $Q$ is a function $Q \colon \mathcal{F} \to \reals^D$ where $D \geq 1$ is the dimension of the query.  For example, in mean estimation we have $D=1$.  We view a query as a statistical property of the distribution $F$.  After being posed the query, the algorithm returns an output $z$ that can depend on $T$ and the state $S_T$.  We write $A_{out}(S_T, T)$ for the output of algorithm $A$ given $T$ and $S_T$.

The algorithm's goal is to minimize the error of the output with respect to a given distance metric $\textsc{dist}(\cdot,\cdot)$.  That is, the algorithm's expected loss for a given choice of $Q$, $F$, and $T$ is $L(A;Q,F,T) = E[\textsc{dist}(Q(F), A_{out}(S_T,T))]$ where the expectation is with respect to the realization of items and any randomness in the algorithm.  We say that the algorithm achieves error $\epsilon$ for a given $T$ if its expected loss is at most $\epsilon$ in the worst case over $Q \in \mathcal{Q}$ and $F \in \mathcal{F}$.

\paragraph{Adaptive Subsampling}
We focus on algorithms that retain no state beyond a subsample of the previously-observed data.  We call these \emph{subsampling algorithms}.  For such an algorithm, the state $S_t$ is an (ordered) subset of $\{x_1, \dotsc, x_t\}$.  We emphasize that because $S_t$ fully describes the state of the algorithm following round $t$, the algorithm cannot retain any other information about the stream beyond the subset retained.  We define $S_0 = \emptyset$ for notational convenience.  

We say that a subsampling algorithm $A$ satisfies the \emph{$m$-recency property} (or just \emph{$m$-recency}) if, in each round $t \geq m$, $S_t \subseteq \{x_t, x_{t-1}, \dotsc, x_{t-m+1}\}$.  That is, $S_t$ is always a subset of the $m$ most-recently-seen items.  Note that an $m$-recent algorithm must have $|S_t| \leq m$; this motivates us to sometimes refer to $m$ as the \emph{memory} of the algorithm.
We say that the algorithm satisfies $\emph{recency}$ when $m$ is clear from context.

\subsection{Framework Examples: Mean Estimation and Linear Regression}
\label{sec:model.examples}

\paragraph{Mean Estimation} The set of potential data items is $\mathcal{X} = \reals^d$ where $d \geq 1$ is the data dimension.  Let $\theta \in \reals^d$ denote the mean of distribution $F$.  The query posed at time $T$ is to estimate $\theta$: i.e., $\mathcal{Q} = \{Q\}$ where $Q(F) = E[F] = \theta$.  The algorithm's output at time $T$ is an estimate $\hat{\theta}$ of the mean, and the algorithm's goal is to minimize $E[||\theta - \hat{\theta}||_2^2]$, the expected squared $\ell_2$ distance. 

To define $\mathcal{F}$ we will make two further assumptions on distribution $F$.  Write $F_j$ for the marginal distribution of coordinate $j$.  Then for each $j$, we assume that there exist constants $\sigma$, $\alpha$, and $\gamma$ such that $\textsc{var}(F_j) \leq \sigma^2$ and $F_j$ has density at least $\alpha$ on the interval $[\theta_j - \gamma, \theta_j + \gamma]$. These assumptions essentially state that the distribution has bounded second moments and is well-behaved around its mean.  For example, these assumptions are satisfied by any correlated Gaussian distribution $N(\theta, \Sigma)$ such that $\Sigma$ has eigenvalues lying in bounded range $[\lambda_0, \lambda_1]$ with $\lambda_0 > 0$.

\paragraph{Linear Regression}  The set of data items is $\mathcal{X} = \reals^d \times \reals$ for some $d \geq 1$.  We will write $(x_t,y_t)$ for a data item, where $x_t \in \reals^d$ is the vector of predictors and $y_t \in \reals$ is the outcome.  Each $x_t$ is drawn from a distribution $G$ over $[0,B]^d$, and then $y_t$ is generated as $y_t = \langle\theta, x\rangle + \epsilon$ where $\langle\cdot, \cdot\rangle$ denotes inner product, $\epsilon \sim N(0,\sigma^2)$ is mean-zero Gaussian noise, and $\theta \in \reals^d$ is a collection of coefficients.\footnote{We can easily extend this to allow $y$ to be multidimensional as well, treating each coordinate of $y$ as a separate single-dimensional regression task.}

We will make the standard assumption that distribution $G$ is well-conditioned in the following sense.  Given a number of rounds $t$, write $X \in \reals^{t \times d}$ for the empirical design matrix whose rows consist of the items $x_1, \dotsc, x_t$.  Then the empirical second moment matrix $\frac{1}{t}X^T X$ is positive semi-definite, and for all $t \geq d$ we require that the distribution $G$ is such that $X^T X$ is invertible with probability $1$.  Moreover, $E[ (x^T x) ]$ has bounded spectral norm, with eigenvalues lying in a bounded range $[\lambda_0, \lambda_1]$ with $\lambda_0 > 0$.  For example, this rules out degenerate scenarios where the vectors $x_j$ lie in a lower-dimensional subspace of $\reals^d$ and hence the distribution has zero density.

The query posed at time $T$ will be a prediction query, described by some $x \in \reals^d$ with $||x||_2^2 \leq 1$.  The desired output is $\langle\theta, x\rangle$.  That is, we can write $\mathcal{Q} = \{Q_x \colon x \in \reals^d, ||x||_2^2 \leq 1 \}$ where $Q_x(F) = \langle\theta, x\rangle$ recalling that $\theta$ is a parameter determined by $F$.  The algorithm outputs a prediction $\hat{y} \in \reals$, and seeks to minimize $E[|\hat{y} - Q_x(F)|^2]$.  We say that the algorithm achieves $\ell_2$ error (or risk) $\epsilon$ if the expected squared prediction error is at most $\epsilon$ in the worst case over all $F$ and $Q_x$.  We emphasize that the query vector $x$ need not be drawn from the training distribution $G$, so prediction error is evaluated in the worst case rather than with respect to a distribution over the query $Q_x$.

One could additionally ask for the stronger requirement that the algorithm achieve low error on the learned regression parameters, rather than low expected prediction error.  That is, one could take $\mathcal{Q} = \{Q\}$ where $Q(F) = \theta$, and then evaluate error as $E[||\hat{\theta} - \theta||_2^2]$.  
The algorithms we construct in this paper will actually satisfy this stronger notion of approximation.

\subsection{A Batch Formulation}

Any algorithm that satisfies $m$-recency  can be thought of in the following nearly-equivalent way, which we call the batched model.  The data stream provides data in \emph{batches} of $m$ items.  In each round, $m$ new items arrive, each drawn independently from distribution $F$.  Write $M_t$ for the set of samples that arrive in round $t$.  The algorithm then updates its state $S_t$ to be a subset of $M_t$.  The choice of $S_t \subseteq M_t$ can depend arbitrarily on $t$, $M_t$, and the previous state $S_{t-1}$.  After the final round $T$, the algorithm generates an output based on the final state $S_T$.

Any algorithm for the batched model can simulate an algorithm with $m$-recency in the streaming model, and vice-versa, using a constant factor additional memory.  This will be helpful throughout the paper, as it will often be convenient to design and describe algorithms in the batched model.  The full proof appears in the appendix.

\begin{proposition}
\label{prop.batch}
Suppose $A$ is an algorithm with $m$-recency for the streaming model that achieves error $\epsilon$ after $T$ rounds.  Then there exists an algorithm $A'$ for the batched model with $m$ memory that achieves error $\epsilon$ in $\lceil T/m \rceil$ rounds. 

Suppose $A$ is an algorithm for the batched model with $m$ memory that achieves error $\epsilon$ after $T$ rounds.  Then there exists an algorithm $A'$ with $2m$-recency that achieves error $\epsilon$ in $mT$ rounds.  
\end{proposition}

\section{Mean Estimation}
\label{sec:mean}

In this section we focus on the mean estimation problem described in the previous section.  For this problem, the optimal estimator for the mean requires $\Omega(1/\epsilon)$ data samples to achieve expected error $O(\epsilon)$ in the worst case over distributions in our class $\mathcal{F}$.  Our main result in this section is an  algorithm that achieves this same asymptotic error rate while satisfying $m$-recency, with $m = O(\log^2(1/\epsilon))$.

\begin{theorem}\label{thm:mean}
    Fix any $\epsilon > 0$ and dimension $d \geq 1$.  There exists a subsampling algorithm $A$ with recency such that if $T > C_1 d/\epsilon$ and $m \geq C_2 d \log(d/\epsilon)$, where $C_1$ and $C_2$ are constants that depend on $\sigma, \gamma, \alpha)$ (parameters of the data distribution) then the expected squared loss is at most $\epsilon$.  The update in each round, as well as generating the final output, can each be performed in time $\textsc{Poly}(d, 1/\epsilon)$.
\end{theorem}

\subsection{Algorithm description}
\label{sec:mean.alg.simple}

Before presenting our algorithm that satisfies the conditions of Theorem~\ref{thm:mean} we will first describe a more straightforward algorithm, listed as Algorithm~\ref{alg:mean_sgd}, based on stochastic gradient descent.  Algorithm~\ref{alg:mean_sgd} is equivalent to our final algorithm in the single-dimensional case $d=1$, but for $d > 1$ our analysis of Algorithm~\ref{alg:mean_sgd} requires additional assumptions on the input distribution, has suboptimal dependence on the dimension $d$, and requires execution time that is exponential in $d$.  We will show how to modify Algorithm~\ref{alg:mean_sgd} to satisfy the conditions of Theorem~\ref{thm:mean} in Section~\ref{sec:mean.alg.improved}.

\begin{algorithm}
\caption{Simple Subsampling for Mean Estimation}\label{alg:mean_sgd}

Input: learning rate $\eta_t > 0$, stream of data batches $M_1, M_2, \dotsc, M_T$ each containing $m$ elements of $\reals^d$
\vspace{2mm}

Round $1$ initialization: $S_1 \leftarrow M_1$ \\

\For{rounds $t = 2, 3, \cdots , T$}{
    $s_{t-1} \leftarrow \textsc{avg}(S_{t-1})$\;

    Arbitrarily partition $M_t = R_t \cup N_t$ with $|R_t| = |N_t| = m/2$\;

    $y_t \leftarrow \textsc{avg}(R_t)$\;

    $z_t \leftarrow s_{t-1} + \eta_t(y_t - s_{t-1})$\;

    Choose $S_t \in \arg\min_{S \subseteq N_t}\{ ||z_t - \textsc{avg}(S)||_2 \}$\;
}

\Return $\textsc{avg}(S_T)$

\end{algorithm}

Algorithm~\ref{alg:mean_sgd} is a subsampling algorithm in the batched model that attempts to simulate the progression of stochastic gradient descent.  Recall that at the beginning of each round $t$, a batched algorithm has a state $S_{t-1}$ that is a subset of data samples $M_{t-1}$ observed in the previous round.  We will write $s_{t-1}$ for the average of the samples in $S_{t-1}$.  In round $t$ a new set $M_t$ of data items arrive, with $|M_t| = m$.  We will partition $M_t$ arbitrarily into two sets, $R_t$ and $N_t$, with $|R_t| = b \geq 1$ (where $b$ is a parameter that will be set later to $m/2$).  The set $R_t$ is used to calculate a gradient that will guide our subsampling.  To this end, the learner calculates the average of the items in $R_t$, call this $y_t$.  We think of $(s_{t-1}-y_t)$ as the proposed gradient.  We then set $z_t = s_{t-1} + (y_t-s_{t-1})\eta_t$ where $\eta_t > 0$ is the (round-dependent) learning rate of our procedure.  We think of $z_t$ as a target state.

In standard SGD one would update the state by setting $s_t = z_t$.  In our subsampling algorithm we are unable to store $z_t$ directly.  Rather, we will use $z_t$ to guide our choice of subsample which will be taken from $N_t$.  Specifically, we will choose $S_t$ to be whichever subset of $N_t$ has average closest to $z_t$ (in $\ell_2$ distance).  The process then repeats with the next round.  When the final round terminates, the algorithm returns the average of the items in $S_T$ as its estimation of the mean.

Notice that Algorithm 1 satisfies $m$-recency: After each batch, the algorithm ``deletes'' all of the data points appearing in the previous batch. However, the choice of state $S_t \subseteq M_t$ can encode information about the previous state $S_{t-1}$, including the previous batch $M_{t-1}$ (and by induction, all other previous batches). In this sense, the data is not truly deleted, and classical notions of differential privacy can still be violated despite the algorithm not explicitly retaining the data.\footnote{A concrete example of this is provided in Appendix~\ref{app:df}. For an intuitive understanding, consider the mean estimation problem in one dimension. One can imagine a situation where the first batch has a data point $x^*$ that is an extreme positive outlier, so the $S_1$ initialization will lead to a large positive mean for $\textsc{avg}(S_1)$. In round 2, while $x^*$ will be deleted, $S_2$ will be chosen to be a singleton consisting of the largest element of $N_2$, $S_2 = \{\tilde{x}\}$. In round 3, $S_3$ will be chosen to consist of the largest elements of $N_3$ because $\tilde{x}$ will itself be a large positive number, and so on. One can see that $\textsc{avg}(S_t)$ will remain large for many batches after the original data point $x^*$ is deleted.   }

\subsection{Analysis of Algorithm~\ref{alg:mean_sgd}}

Our approach will be to relate the progression of Algorithm~\ref{alg:mean_sgd} with that of standard stochastic gradient descent.  To see why this is helpful, imagine that we were able to set $S_t = \{z_t\}$ at the conclusion of each round in Algorithm~\ref{alg:mean_sgd}.  Setting the learning rate $\eta_t = 1/t$, classic SGD analysis would imply that after $T$ rounds the expected squared loss would be $O(1/T)$~\cite{rakhlin2012making}. 

Since we cannot set $S_t = \{z_t\}$, we will instead choose $S_t$ so that $\textsc{avg}(S_t)$ is close to $z_t$ in $\ell_2$ distance.  This introduces small perturbations to each gradient step in our algorithm.  We will treat these perturbations as adversarial noise in the update step.  In Lemma~\ref{lem:mean.sgd} we show that the SGD analysis can be made robust to such noise.  We then show in Lemma~\ref{lem:mean.errors} that, with high probability, the distance between $\textsc{avg}(S_t)$ and $z_t$ will be sufficiently small in each round as long as $m$ is large enough. Combining these ingredients will allow us to bound the error obtained by Algorithm~\ref{alg:mean_sgd}.

We first review the standard setting of SGD.  There are $T$ rounds and a convex function $H \colon \reals^d \to \reals$ minimized at $\theta \in \reals^d$, say with $H(\theta) = 0$. We require that $H$ is $\lambda$-strongly convex, meaning that for all $w,w' \in \reals^d$ and $g$ a gradient of $H$ at $w$, we have $H(w') \geq H(w) + \langle g,w'-w \rangle + \frac{\lambda}{2}||w'-w||_2^2$. 

The SGD procedure initializes some point $w_0 \in \reals^d$.  Then in each round $t \geq 1$ we receive a random $\hat{g}_t$ from a gradient oracle such that $E[\hat{g}_t] = g_t$, where $g_t$ is the gradient of $H$ at $w_{t-1}$.  The SGD algorithm then applies update rule $w_t = w_{t-1} - \eta_t \hat{g}_t$ and repeats.  In our variant with adversarial noise, the update rule will be $w_t = w_{t-1} - \eta_t \hat{g}_t + \zeta_t$, where $\zeta_t$ is possibly randomized and can be drawn from an adversarially-chosen distribution in each round.  As long as the adversarial noise is on the same order of magnitude as the second moment of the gradients, this does not substantially inflate the convergence time of the SGD procedure.  A full proof appears in the appendix.

\begin{lemma}
\label{lem:mean.sgd}
Suppose objective function $H$ is $\lambda$-strongly convex and that $E[||\hat{g}_t||_2^2] \leq \Gamma^2$.  Suppose that we employ SGD with adversarial noise as described above, and suppose that $E[||\zeta_t||_2^2] \leq \Gamma^2 / (\lambda^2 T^3)$ in each round $t$. Then after $T$ steps we have $E[||w_T - \theta||^2] \leq 7\Gamma^2 / \lambda^2 T$.
\end{lemma}

Our plan is to employ Lemma~\ref{lem:mean.sgd} to analyze Algorithm~\ref{alg:mean_sgd}. 
 Recall that we define $z_t = s_{t-1} + (y_t-s_{t-1})\eta_t$.  We will treat $(s_{t-1} - y_t)$ as a random variable whose expectation is the gradient of function $H(s) = \frac{1}{2}||s-\theta||_2^2$ evaluated at $s_{t-1}$.  We treat the difference between $z_t$ and $s_t$ (our estimate of $z_t$) as adversarial noise, taking the role of $\zeta_t$ in Lemma~\ref{lem:mean.sgd}.  To bound the size of this noise, we will make use of the following minor reformulation of a result from~\cite{becchetti2022multidimensional}.

\begin{lemma}[Based on Corollary 24 of~\cite{becchetti2022multidimensional}]
\label{lem:mean.errors}
Given $\epsilon > 0$, suppose $m$ vectors $x_i \in \mathbb{R}^d$ are drawn iid from a $d$-dimensional Gaussian $N(\theta, \Sigma)$ where the eigenvalues of $\Sigma$ lie in $[(\sigma/a)^2, \sigma^2]$ for some $a \geq 1$.  Then there exists a constant $C > 0$ such that if $m \geq C a d^2 \log^2(d\sigma/a\epsilon)$ then, for any $z \in [\theta-(\sigma/a), \theta+(\sigma/a)]^d$, with probability at least $1 - \epsilon/2$, there exists a subset $S$ of the vectors with average $\textsc{avg}(S)$ such that $||z - \textsc{avg}(S)|| < 2\epsilon$.
\end{lemma}

Combining Lemmas~\ref{lem:mean.sgd} and~\ref{lem:mean.errors} yields a weakened version of Theorem~\ref{thm:mean} for Algorithm~\ref{alg:mean_sgd}.

\begin{proposition}\label{prop:mean.simple}
Fix $\epsilon > 0$ and suppose $F = \mathcal{N}(\theta, \Sigma)$ where all eigenvalues of $\Sigma$ lie in $[(\sigma/a), \sigma]$ where $a \geq 1$.  Then if we run Algorithm~\ref{alg:mean_sgd} for $T$ rounds and $m$ memory with $T > 12d\sigma^2/m\epsilon$ and $m \geq C d^2 \log^2(d\sigma/a\epsilon) )$, then the expected squared loss after round $T$ is at most $\epsilon$.
\end{proposition}

\subsection{An Improved Algorithm}
\label{sec:mean.alg.improved}

As part of our analysis of Algorithm~\ref{alg:mean_sgd}, it was necessary to find a subset of $d$-dimensional vectors whose average is very close to a target vector in $\ell_2$ distance.  In a sense, this subset must be coordinated across the dimensions of the input vector.  When the input vectors are drawn from Gaussian distributions, it is possible to find such a coordinated subset among $O(d^2 \log^2(d/\epsilon))$ vectors.  But we can improve this memory requirement and extend beyond Gaussians. It turns out that the random subset problem admits the following improved solution for the single-dimensional case:

\begin{theorem}[Theorem 2.4 of \cite{Lueker:1998}]\label{thm:subset.sum}
Suppose that $x_1, \dotsc, x_m$ are drawn independently from $U[-1,1]$.  Then if $m > C_2 \log(1/\epsilon)$ then with probability at least $1 - e^{-C_1 m}$ there is a subset $S$ of the items with $|z-\sum_{x \in S}x| < \epsilon$, for any $z \in [-1/2, 1/2]$.
\end{theorem}

This theorem suggests the following improved analysis for Algorithm~\ref{alg:mean_sgd} in the case where $d=1$.  Recall that we assume distribution $F$ has density at least $\alpha$ on the interval $[\theta - \gamma, \theta + \gamma]$.  Choose any $z \in [\theta - \gamma/2, \theta + \gamma/2]$.  Then $F$ has density at least $\alpha$ on the interval $[z-\gamma/2, z+\gamma/2]$.  We can therefore think of distribution $F$ as drawing a value uniformly from $[z-\gamma/2, z+\gamma/2]$ with probability $\alpha\gamma$, and drawing a value from some arbitrary residual distribution (with bounded variance) otherwise.  As long as $m > 2 C_2 (\alpha\gamma)^{-1} \log(1/\epsilon)$, Chernoff bounds imply that at least $C_2 \log(1/\epsilon)$ items will be drawn from this uniform distribution with probability $1-e^{-C_3 m}$ for some constant $C_3$.  For each of those items $x_i$, note that $(x_i - z)$ is uniformly distributed on $[-\gamma/2, \gamma/2]$.  Then by Theorem~\ref{thm:subset.sum}, with probability at least $1 - e^{-C_1 C_2 \log(1/\epsilon)}$ there will be a subset $S$ of those items such that $|\sum_{x \in S}(x-z)| < \gamma\epsilon/2$.  This means that $|\textsc{avg}(S)-z| < \gamma\epsilon/2 < \epsilon$ as well.

Employing these bounds in the place of Lemma~\ref{lem:mean.errors} in our proof of Proposition~\ref{prop:mean.simple}, we conclude that if $m > C_2 (\alpha\gamma)^{-1} \log(1/\epsilon)$, Algorithm~\ref{alg:mean_sgd} will have expected squared error at most $2\epsilon$ after $T = \Theta(\sigma^2/m\epsilon)$ rounds, for any (possibly non-Gaussian) distribution $F$ satisfying our conditions.  Moreover, note that updates under this algorithm can be computed in time $2^m = \textsc{Poly}(1/\epsilon, 2^{(\alpha\gamma)^{-1}})$, as this is the time required to check all $2^m$ subsets of items.  For $\alpha, \gamma = O(1)$ this is $\textsc{Poly}(1/\epsilon)$.

We can extend this improvement to the $d$-dimensional case as follows.  For each dimension $i$ from $1$ to $d$, simulate (in parallel) a separate single-dimensional instance of Algorithm~\ref{alg:mean_sgd} to estimate $\theta_i$, each with its own separate segment of memory.  This inflates the memory requirement by a factor of $d$, so we require $m > C_2 d (\alpha\gamma)^{-1} \log(1/\epsilon')$ to achieve squared error at most $\epsilon'$ in each dimension.  The resulting (total) squared error will be at most $d\epsilon'$, so to achieve error $\epsilon$ we will perform a change of variables $\epsilon' = \epsilon/d$, resulting in memory requirement $m > C_2 d (\alpha\gamma)^{-1} \log(d/\epsilon)$ and time requirement $T = \Theta(d \sigma^2/m\epsilon)$.  This setting of parameters satisfies the requirements of Theorem~\ref{thm:mean}.

\subsection{Lower Bound for Mean Estimation}
\label{sec:mean.lowerbound}

We next show that the bound we obtained on our algorithm's recency condition, $m$, in Theorem~\ref{thm:mean} is almost tight asymptotically with respect to $d$ and $\epsilon$.

\begin{theorem}\label{thm:lowerbound}
Suppose data items are drawn from a standard normal distribution $\mathcal{N}(\theta, I_d)$ (i.e., with each dimension drawn from an independent Gaussian with unit variance).  Choose any $\epsilon > 0$.  If $m < \frac{d\log(1/\epsilon)}{\log d + \log\log(1/\epsilon)}$ then for any algorithm that satisfies recency and any $T$, with probability at least $2/3$ the squared error will be strictly greater than $\epsilon$.
\end{theorem}

The formal proof of Theorem~\ref{thm:lowerbound} appears in the appendix.  The high-level idea is as follows.  We will allow the algorithm designer to define an arbitrary output function $f$ that maps the final subset to an estimate $\hat{\theta}$.  Once this function is chosen, we will actually reveal $\theta$ to the algorithm and allow the subsampling procedure to use its knowledge of $\theta$ to choose which subset of data points to retain each round.  Under this assumption the behavior of the algorithm before round $T$ is irrelevant, and the only question is whether there exists any subset $S \subseteq M_T$ of the final set of data items such that $||f(S) - \theta||_2^2 < \epsilon$.  However, for any \emph{fixed} subset $S$, we can think of $\theta$ as being drawn from the maximum likelihood distribution given $S$, which is a Gaussian centered at the empirical mean of $S$.  This means that $\Pr[||f(S) - \theta||_2^2 < \epsilon]$ is maximized at the empirical mean estimator, for which we can calculate this probability precisely.  We can then take a union bound over all $2^m$ subsets of $M_t$ to bound the probability that \emph{any} subset $S$ satisfies $||f(S) - \theta||_2^2 < \epsilon$.  Setting that probability bound to a constant results in the claimed bound on $m$.

\section{Linear Regression}
\label{sec:regression}

We next show how to extend our subsampling algorithm for mean estimation to probabilistic linear regression problems.  We recall that this class of problems was described in Section~\ref{sec:model.examples}.

\begin{theorem}\label{thm:regression}
Fix any $\epsilon > 0$ and consider the linear regression task.  There exists a subsampling algorithm $A$ with recency such that if $T > C_1 d/\epsilon$ and $m \geq C_2 d^2 \log(d) \log(d/\epsilon)$, where $C_1$ and $C_2$ are constants that depend on $\sigma, B, \lambda_0, \lambda_1$ (parameters of the data distribution) then the expected squared loss at most $\epsilon$.  The update in each round, as well as generating the final output, can each be performed in time $\textsc{Poly}(d, 1/\epsilon)$.
\end{theorem}

The proof of Theorem~\ref{thm:regression} appears in Appendix~\ref{app:proofs.regression}; here we comment on the approach and challenges.  Like mean estimation, online probabilistic linear regression has a standard stochastic gradient descent (SGD) solution using the least squares method.  Roughly speaking, one maintains a state corresponding to an estimate $\hat{\theta}_t$ of the regression coefficients $\theta$, and refines $\hat{\theta}_t$ via gradient descent evaluated on batches of input samples. We would ideally implement this SGD approach by finding a subset of data items that jointly approximate a target estimate, as we did for mean estimation.  But there is a complication: we only have access to subsets of $(x,y)$ pairs that do not directly translate into noisy estimates of $\theta$.  Our strategy will be to store the data items $(x,y)$ in groups of a fixed size $k \geq d$, and associate each group with the maximum likelihood estimate (MLE) it induces for $\theta$.  We will then think of each group's MLE as a noisy estimate of $\theta$, and search for a subset of these estimates whose average is close to our target update.  We must choose $k$ large enough that these estimates are sufficiently well-behaved; it turns out that taking $k = \Omega(d \log d)$ will suffice.  Putting these ideas together will yield Theorem~\ref{thm:regression}.

\subsection{Algorithm}

We now describe our algorithm in more detail.  The procedure is listed as Algorithm~\ref{alg:regression}.  As we did for the improved algorithm for mean estimation in Section~\ref{sec:mean.alg.improved}, we will interpret the state $S_t$ retained between each round as a disjoint union of $d$ subsets $S_{t,1}, \dotsc, S_{t,d}$, corresponding to the $d$ dimensions of $\theta$.  Subset $S_{t,i}$ encodes an estimate for $\theta_i$.  The algorithm begins each round by recovering the estimate $s_i$ of $\theta_i$.  This is done using the $\textsc{Decode}$ subroutine, which splits a given set of data items into groups of size $k$, recovers a maximum likelihood estimator for $\theta$ from each group, then averages those estimators to return an aggregate estimation.  We decode each $S_{t,i}$ separately and use only coordinate $i$ from $S_{t,i}$, discarding the remainder.

\SetKwFunction{FDecode}{Decode}
\SetKwProg{Fn}{Function}{:}{}
  
\begin{algorithm}
\caption{Subsampling for Linear Regression}\label{alg:regression}

Input: group size $k \geq d$, learning rate $\eta_t > 0$, stream of data batches $M_1, M_2, \dotsc, M_T$ each containing $m$ elements $(x,y) \in \reals^d \times \reals$
\vspace{2mm}

State: ordered set $S_t = S_{t,1} \cup \dotsc \cup S_{t,d}$ with $|S_{t,i}| \leq m/d$ for each $i \in [d]$.

\medskip
Round $1$ initialization: $S_1 \leftarrow M_1$ \\

\For{rounds $t = 2, 3, \cdots , T$}{

    \medskip
    \tcc{Step 1: Recover Previous Estimate $s_{t-1}$}

    \For{$i = 1, 2, \dotsc, d$}{
        $[s_{t-1}]_i \leftarrow [{\textsc{Decode}}(S_{t-1,i})]_i$\;
    }

    \medskip
    \tcc{Step 2: Generate Target Update $z_t$}

    Arbitrarily partition $M_t = R_t \cup N_t$ with $|R_t| = |N_t| = m/2$\;

    Let $X \in \reals^{b \times d}$ and $Y \in \reals^{b}$ be the matrix representation of the items in $R_t$\;

    $z_{t} \leftarrow s_{t-1} + \eta_t(X^T Y - X^T X s_{t-1})$\;

    \medskip
    \tcc{Step 3: Encode Approximation to $z_t$}

    Partition $N_t$ into $d$ subsets $N_{t,1}, \dotsc, N_{t,d}$ of equal size\;

    \For{$i = 1, 2, \dotsc, d$}{

        Partition $N_{t,i}$ into subsets $L_1, \dotsc, L_r$ each of size $k$\;
        
        Let $w_j \leftarrow \textsc{Decode}(L_j)$ for all $j \in [r]$\;
        
        Choose $P \subseteq [j]$ that minimizes $|\ [z_t]_i - [\textsc{avg}(\{w_j \colon j \in P\}]_i\ |$\;

        Let $S_{t,i} = \cup_{j \in P} L_j$\;
    }    
    $S_t \leftarrow \cup_i S_{t,i}$\;
      
}

\medskip
\tcc{Output: Recover Estimate $s_T$}
\For{$i = 1, 2, \dotsc, d$}{
    $[s_T]_i \leftarrow [{\textsc{Decode}}(S_{T,i})]_i$\;
}
\Return $s_T$

\bigskip
\Fn{\FDecode{ordered set $S$ of data items}}{
    Partition $S \leftarrow L_1 \cup L_2 \cup \dotsc \cup L_r$ where each $L_j$ has size $k$\;
    \For{$j = 1,\dotsc,r$}{
        Let $X \in \reals^{b \times d}$ and $Y \in \reals^{b}$ be the matrix representation of the items in $S$\;
        $\theta_j \leftarrow (X^T X)^{-1}X^T Y$\;
    }
    \Return $\textsc{avg}(\{\theta_1, \dotsc, \theta_r\})$\;
}

\end{algorithm}

Once $s_{t-1}$ has been recovered, we use half of the newly arriving data items to simulate a step of stochastic gradient descent via ordinary least squares.  This results in a target $z_t$.

Finally, we encode an approximation to $z_t$ using a subset of the remaining data items $N_t$.  We approximate each coordinate of $z_t$ separately, using a disjoint subset of $N_t$ for each.  For each coordinate $i$, we first group the data items into group of size $k$, then use the subroutine $\textsc{Decode}$ to recover the maximum likelihood estimate of $\theta$ from each group.  We then find the subset of these MLEs whose average, in coordinate $i$, best approximates $[z_t]_i$, the $i$th coordinate of the target.  We will retain that subset of groups as $S_{t,i}$.

Finally, at the end of round $T$, we employ the same decoding procedure from the beginning of each round to recover our estimate of $\theta$ from $S_T$.  This estimate is then returned by the algorithm.

\subsection{Analysis}

Our analysis of Algorithm~\ref{alg:regression} closely follows that of the improved algorithm for Mean Estimation from Section~\ref{sec:mean.alg.improved}.  Writing $s_t = \textsc{Decode}(S_t)$ for the estimate of $\theta$ maintained in round $t$, we will invoke Lemma~\ref{lem:mean.sgd} to analyze the expected squared distance between $s_t$ and $\theta$ over time.  This has two parts.  First we note that the target $z_t$ computed in round $t$ encodes a step of stochastic gradient descent, starting from $s_{t-1}$, without noise. We then treat the encoding error $||s_t - z_t||_2^2$ as adversarial noise and we bound its expectation.

The key technical challenge is to bound the encoding error $||s_t - z_t||_2^2$.  It is here where we will make use of our choice to group data items into collections of size $k$.  We will show that as long as $k$ is sufficiently large, matrix Chernoff bounds imply that the estimator $(X^T X)^{-1}X^T Y$ of $\theta$ for a group of data items will be smoothly distributed around $\theta$.  This will enable us to use Theorem~\ref{thm:subset.sum} to find a subset of these estimators whose average is very close to the target point $z_t$.

\begin{lemma}
\label{lem:matrix.concentration}
There exist constants $C$, $\alpha$, and $\gamma$ depending on $\sigma^2$ (the variance of the noise distribution) and the distribution $G$ over data items such that the following holds. Let $(x_1, y_1), \dotsc, (x_k,y_k)$ be independent draws from our data generating process, where $k > C d \log d$.  Write $\theta' = (X^T X)^{-1}X^TY$, where $X$ and $Y$ are the matrices of predictors and outcomes. Then for any $i \in [d]$, the distribution over $\theta'_i$ has density at least $\alpha$ in the range $[\theta_i - \gamma, \theta_i + \gamma]$.
\end{lemma}

With Lemma~\ref{lem:matrix.concentration} in hand, we can employ the same analysis we used to prove Theorem~\ref{thm:mean} using our improved algorithm for learning the mean in $d$ dimensions, treating each group of $k$ input samples as a single sample of $\theta$.  The formal proof of Theorem~\ref{thm:regression}, which follows the proof of Theorem~\ref{thm:mean}, appears in the appendix.

\section{Conclusions and Future Work}
\label{sec:conclusion}

In this work we introduced a framework for online algorithms subject to strict data retention limits. The algorithms in our framework retain no state other than a subsample of the data, and each data point must be removed from the subsample after at most $m$ rounds. We provide upper and lower bounds on the value of $m$ needed to achieve error $\epsilon$ for mean estimation and linear regression. We find that it is possible to substantially outperform a naive maximal-storage baseline by adaptively and proactively curating the algorithm's dataset in order to improve its representativeness of the full data stream (including data that was to have been dropped).

Many technical questions are left open for future pursuits. We take a worst-case perspective that all data must be removed after $m$ rounds, but one might consider a model where some data points can be retained for much longer.  Does the presence of long-lived data alongside data that must be removed quickly enable different algorithmic approaches? Our subsampling framework can also be extended to other statistical tasks like non-linear regression, estimating higher moments, classification tasks, and so on.  In each case, the algorithmic challenge is to dramatically reduce the size of a training set, online, so that a (perhaps specially-tailored) training process executed on the subsample can achieve performance approximately matching what is possible on the full data.

One could also apply our framework to non-stochastic or partially stochastic environments, where data is not necessarily generated according to a stationary process.  Such environments can amplify the impact of individual data points (and their removal) on an algorithm's output and state.  Of course, the achievable algorithmic guarantees might vary substantially depending on the assumptions made on the data.  But even so, understanding the structure of optimal (or near optimal) algorithms can shed light on the manner in which algorithm designers may be incentivized to build systems in the face of data retention limitations.

Finally, one can explore whether alternative frameworks for data removal lead to different types of behavior in optimal algorithm designs.  A step in this direction is to quantify the extent to which optimal algorithms in a given framework are ``undesirable'' from the perspective of data removal, and use this to directly compare frameworks.  Such an endeavor can help to build a toolkit for building up algorithmic restrictions that align well with stated policy goals. 

\bibliographystyle{alpha}
\bibliography{bibliography}

\newcommand{\etalchar}[1]{$^{#1}$}
\begin{thebibliography}{DMIMW12}

\bibitem[BdCC{\etalchar{+}}22]{becchetti2022multidimensional}
Luca Becchetti, Arthur Carvalho~Walraven da~Cuhna, Andrea Clementi, Francesco d'Amore, Hicham Lesfari, Emanuele Natale, and Luca Trevisan.
\newblock On the multidimensional random subset sum problem.
\newblock {\em arXiv preprint arXiv:2207.13944}, 2022.

\bibitem[BLK17]{bachem2017practical}
Olivier Bachem, Mario Lucic, and Andreas Krause.
\newblock Practical coreset constructions for machine learning.
\newblock {\em arXiv preprint arXiv:1703.06476}, 2017.

\bibitem[Bot10]{bottou2010large}
L{\'e}on Bottou.
\newblock Large-scale machine learning with stochastic gradient descent.
\newblock In {\em Proceedings of COMPSTAT'2010: 19th International Conference on Computational StatisticsParis France, August 22-27, 2010 Keynote, Invited and Contributed Papers}, pages 177--186. Springer, 2010.

\bibitem[CCP18]{CCPA}
California consumer privacy act.
\newblock \url{https://leginfo.legislature.ca.gov/faces/codes_displayText.xhtml?division=3.&part=4.&lawCode=CIV&title=1.81.5}, 2018.
\newblock Cal. Civ. Code §§ 1798.100 et seq.

\bibitem[CDP21]{CDPA}
Consumer data protection act, 2021 h.b. 2307/2021 s.b. 1392.
\newblock \url{https://lis.virginia.gov/cgi-bin/legp604.exe?ses=212&typ=bil&val=Hb2307}, 2021.

\bibitem[CSSV23]{cohen2023control}
Aloni Cohen, Adam Smith, Marika Swanberg, and Prashant~Nalini Vasudevan.
\newblock Control, confidentiality, and the right to be forgotten.
\newblock In {\em Proceedings of the 2023 ACM SIGSAC Conference on Computer and Communications Security}, pages 3358--3372, 2023.

\bibitem[CY15]{cao2015towards}
Yinzhi Cao and Junfeng Yang.
\newblock Towards making systems forget with machine unlearning.
\newblock In {\em 2015 IEEE symposium on security and privacy}, pages 463--480. IEEE, 2015.

\bibitem[dCdG{\etalchar{+}}23]{da2023revisiting}
Arthur da~Cunha, Francesco d'Amore, Fr{\'e}d{\'e}ric Giroire, Hicham Lesfari, Emanuele Natale, Laurent Viennot, et~al.
\newblock Revisiting the random subset sum problem.
\newblock In {\em European Symposium on Algorithms}. Schloss Dagstuhl-Leibniz-Zentrum f{\"u}r Informatik, 2023.

\bibitem[DLFU13]{dhillon2013new}
Paramveer Dhillon, Yichao Lu, Dean~P Foster, and Lyle Ungar.
\newblock New subsampling algorithms for fast least squares regression.
\newblock {\em Advances in neural information processing systems}, 26, 2013.

\bibitem[DMIMW12]{drineas2012fast}
Petros Drineas, Malik Magdon-Ismail, Michael~W Mahoney, and David~P Woodruff.
\newblock Fast approximation of matrix coherence and statistical leverage.
\newblock {\em The Journal of Machine Learning Research}, 13(1):3475--3506, 2012.

\bibitem[Fel20]{feldman2020introduction}
Dan Feldman.
\newblock Introduction to core-sets: an updated survey.
\newblock {\em arXiv preprint arXiv:2011.09384}, 2020.

\bibitem[FTGB22]{ferbach2022general}
Damien Ferbach, Christos Tsirigotis, Gauthier Gidel, and Joey Bose.
\newblock A general framework for proving the equivariant strong lottery ticket hypothesis.
\newblock In {\em The Eleventh International Conference on Learning Representations}, 2022.

\bibitem[GDP16]{GDPR}
General data protection regulation, 2016.
\newblock Regulation 2016/679 of the European Parliament and of the Council of 27 April 2016 on the protection of natural persons with regard to the processing of personal data and on the free movement of such data, and repealing Directive 95/46/EC (General Data Protection Regulation), OJ 2016 L 119/1.

\bibitem[GGV20]{garg2020formalizing}
Sanjam Garg, Shafi Goldwasser, and Prashant~Nalini Vasudevan.
\newblock Formalizing data deletion in the context of the right to be forgotten.
\newblock In {\em Annual International Conference on the Theory and Applications of Cryptographic Techniques}, pages 373--402. Springer, 2020.

\bibitem[HRW11]{hall2011random}
Rob Hall, Alessandro Rinaldo, and Larry Wasserman.
\newblock Random differential privacy.
\newblock {\em arXiv preprint arXiv:1112.2680}, 2011.

\bibitem[Lue98]{Lueker:1998}
George~S. Lueker.
\newblock Exponentially small bounds on the expected optimum of the partition and subset sum problems.
\newblock {\em Random Structures \& Algorithms}, 12(1):51--62, 1998.

\bibitem[MM13]{meng2013low}
Xiangrui Meng and Michael~W Mahoney.
\newblock Low-distortion subspace embeddings in input-sparsity time and applications to robust linear regression.
\newblock In {\em Proceedings of the forty-fifth annual ACM symposium on Theory of computing}, pages 91--100, 2013.

\bibitem[MMY14]{ma2014statistical}
Ping Ma, Michael Mahoney, and Bin Yu.
\newblock A statistical perspective on algorithmic leveraging.
\newblock In {\em International Conference on Machine Learning}, pages 91--99. PMLR, 2014.

\bibitem[NHN{\etalchar{+}}22]{nguyen2022survey}
Thanh~Tam Nguyen, Thanh~Trung Huynh, Phi~Le Nguyen, Alan Wee-Chung Liew, Hongzhi Yin, and Quoc Viet~Hung Nguyen.
\newblock A survey of machine unlearning.
\newblock {\em arXiv preprint arXiv:2209.02299}, 2022.

\bibitem[PRN{\etalchar{+}}20]{pensia2020optimal}
Ankit Pensia, Shashank Rajput, Alliot Nagle, Harit Vishwakarma, and Dimitris Papailiopoulos.
\newblock Optimal lottery tickets via subset sum: Logarithmic over-parameterization is sufficient.
\newblock {\em Advances in neural information processing systems}, 33:2599--2610, 2020.

\bibitem[RSS12]{rakhlin2012making}
Alexander Rakhlin, Ohad Shamir, and Karthik Sridharan.
\newblock Making gradient descent optimal for strongly convex stochastic optimization.
\newblock In {\em Proceedings of the 29th International Coference on International Conference on Machine Learning}, pages 1571--1578, 2012.

\bibitem[TB18]{ting2018optimal}
Daniel Ting and Eric Brochu.
\newblock Optimal subsampling with influence functions.
\newblock {\em Advances in neural information processing systems}, 31, 2018.

\bibitem[XZZ{\etalchar{+}}23]{xu2023machine}
Heng Xu, Tianqing Zhu, Lefeng Zhang, Wanlei Zhou, and Philip~S Yu.
\newblock Machine unlearning: A survey.
\newblock {\em ACM Computing Surveys}, 56(1):1--36, 2023.

\bibitem[Zhu16]{zhu2016gradient}
Rong Zhu.
\newblock Gradient-based sampling: An adaptive importance sampling for least-squares.
\newblock {\em Advances in neural information processing systems}, 29, 2016.

\end{thebibliography}

\appendix

    \newtheoremstyle{TheoremNum}
        {\topsep}{\topsep}              %%% space between body and thm
        {\itshape}                      %%% Thm body font
        {}                              %%% Indent amount (empty = no indent)
        {\bfseries}                     %%% Thm head font
        {.}                             %%% Punctuation after thm head
        { }                             %%% Space after thm head
        {\thmname{#1}\thmnote{ \bfseries #3}}%%% Thm head spec
    \theoremstyle{TheoremNum}
    \newtheorem{theoremnum}{Theorem}
    \newtheorem{propositionnum}{Proposition}
    \newtheorem{lemmanum}{Lemma}

\section{Omitted Proofs from Section 2}

\begin{propositionnum}[\ref{prop.batch}]
Suppose $A$ is an algorithm with recency for the streaming model with $m$ memory that achieves error $\epsilon$ after $T$ rounds.  Then there exists an algorithm $A'$ for the batched model with $m$ memory that achieves error $\epsilon$ in $\lceil T/m \rceil$ rounds. 

Suppose $A$ is an algorithm for the batched model with $m$ memory that achieves error $\epsilon$ after $T$ rounds.  Then there exists an algorithm $A'$ with recency and $2m$ memory that achieves error $\epsilon$ in $mT$ rounds.  
\end{propositionnum}

\begin{proof}
Suppose $A$ is an algorithm with recency for the streaming model with $m$ memory.  We will simulate $A$ using an algorithm $A'$ in the batched model.  In each round $t$, denote the $m$ items that arrive as $(x_{(t-1)m+1}, x_{(t-1)m+2}, \dotsc, x_{tm})$.  The items that arrive across all rounds then correspond to a stream of items $x_1, x_2, \dotsc$.  Write $\hat{S}_i$ for the state of algorithm $A$ after processing item $x_i$ in this stream.  Our algorithm $A'$ will simulate the progression of $A$ on this stream.  Algorithm $A'$ will have the property that after each round $t$, state $S_t$ will equal $\hat{S}_{tm}$.  

We need to specify the output and update functions of $A'$.  The output functions will be the same: $A'_{out} = A_{out}$.  Since $S_t = \hat{S}_{tm}$, this immediately implies the claimed error of $A'$. The update function $A'_{upd}$ is as follows.  Given $M_t = (x_{(t-1)m+1}, x_{(t-1)m+2}, \dotsc, x_{tm})$ and recalling that $S_{t-1} = \hat{S}_{(t-1)m}$, we compute $\hat{S}_k = A_{upd}(x_k, \hat{S}_{k-1}, k)$ for each $k = (t-1)m+1, (t-1)m+2, \dotsc, tm$ in sequence.  Once this is complete, we take the new state to be $S_t = \hat{S}_{tm}$.  Note that since $\hat{S}_{tm} \subseteq \{x_{tm}, \dotsc, x_{(t-1)m+1}\} = M_t$ by recency, we must have $S_t \subseteq M_t$ and hence this choice of $S_t$ is valid.

For the other direction, suppose $A$ is an algorithm for the batched model with $m$ memory.  Given a data stream $x_1, x_2, \dotsc$, we consider the batched input stream defined by $M_k = (x_{(k-1)m+1}, \dotsc, x_{km})$ for all $k = 1, 2, \dotsc$.  Write $\hat{S}_k$ for the state of $A$ after $k$ rounds.  

We will construct a streaming algorithm $A'$ with recency and $2m$ memory that simulates the progression of $A$.  For each round $t = km$, the state $S_t$ of $A$ will equal $\hat{S}_k$.  For $t = km + z$ with $1 \leq z < m$, state $S_t$ will consist of $\hat{S}_k \cup \{x_{km+1}, \dotsc, x_t\}$.  That is, for each round that is not a multiple of $m$, the algorithm $A'$ simply appends the item to the state.  For each round that is a multiple of $m$, say $t = km$, $A'_{upd}$ proceeds by simulating $A$ on state $\hat{S}_{k-1}$ and $M_k = \{x_{(k-1)m+1}, \dotsc, x_{km}\}$, all of which are present in $S_t \cup \{x_t\}$.  This returns $\hat{S}_k \subseteq M_k$, and the algorithm takes $S_t = \hat{S}_k$.  Note that since $\hat{S}_k \subseteq M_k = \{x_{(k-1)m+1}, \dotsc, x_{km}\}$, this update rule satisfies recency even when appending data items for the following $m$ rounds.

For the output rule, we take $A'_{out}$ to be $A_{out}$ applied to the simulated state $\hat{S}_k$ retained in memory, possibly discarding the most recently appended data items if $T$ is not a multiple of $m$.  We then inherit the error rate of $A$ as claimed.
\end{proof}

\section{Omitted Proofs from Section 3}

\begin{lemmanum}[\ref{lem:mean.sgd}]
Suppose objective function $H$ is $\lambda$-strongly convex and that $E[||\hat{g}_t||_2^2] \leq \Gamma^2$.  Suppose that we employ SGD with adversarial noise as described above, and suppose that $E[||\zeta_t||_2^2] \leq \Gamma^2 / (\lambda^2 T^3)$ in each round $t$. Then after $T$ steps we have $E[||w_T - \theta||^2] \leq 7\Gamma^2 / \lambda^2 T$.
\end{lemmanum}

\begin{proof}
Write $L_t = E[||w_t - \theta||^2]$, the expected squared loss of the state in round $t$.  We will show the stronger statement that $L_t \leq 7\Gamma^2 / (\lambda^2 t)$ for all $t \leq T$.

For each $t \geq 1$, we will write $\tilde{L}_{t}$ for the squared distance after a step of SGD dynamics is applied, then $L_{t}$ for the squared distance after adversarial noise is added.  

Standard SGD analysis (see \cite{rakhlin2012making}, Lemma 2) yields $\tilde{L}_1 \leq 4\Gamma^2 / \lambda^2$.  Then if we apply triangle inequality for the adversarial noise, we have $L_1 \leq \tilde{L}_1 + \mu \sqrt{\tilde{L}_{t+1}} + \mu^2 \leq 7G^2/\lambda^2$.

For each $t \geq 1$, if we take $\eta_t = 1/(\lambda t)$, then standard calculations yield
\begin{align*} 
\tilde{L}_{t+1} &= E[||w_t - \eta_t \hat{g}_t - \theta||^2]\\
&= E[||w_t - \theta||^2] - 2\eta_t E[<\hat{g}_t, w_t - \theta>] + \eta_t^2 E[||\hat{g}_t||^2]\\
&= L_t - 2 \eta_t E[<g_t, w_t - \theta>] + \Gamma^2 / (\lambda^2 t^2)\\
&\leq L_t - 2 \eta_t E\left[H(w_t) - H(\theta) + \frac{\lambda}{2}||w_t - \theta||^2\right] + \Gamma^2 / (\lambda^2 t^2)\\
&\leq L_t - 2 \eta_t E\left[\frac{\lambda}{2}||w_t - \theta||^2 + \frac{\lambda}{2}||w_t - \theta||^2\right] + \Gamma^2 / (\lambda^2 t^2)\\
& = (1 - 2/t)L_t + \Gamma^2 / (\lambda^2 t^2).
\end{align*}
Applying triangle inequality for the adversarial noise yields
\[ L_{t+1} \leq \tilde{L}_{t+1} + 2E[||\zeta_t||]\sqrt{\tilde{L}_{t+1}} + E[||\zeta_t||^2] \leq \tilde{L}_{t+1} + 2\mu\sqrt{\tilde{L}_{t+1}} + \mu^2. \]

We will bound $L_t$ for $t \geq 2$ using induction.  For $L_2$, plugging $t=2$ into the inequality $\tilde{L}_{t+1} \leq (1 - 2/t)L_t + \Gamma^2 / (\lambda^2 t^2)$ yields
\[ \tilde{L}_2 \leq \Gamma^2 / 4\lambda^2  \]
and hence
\[ L_2 = \frac{\Gamma^2}{4(\lambda^2)} + 2 \cdot \frac{\Gamma}{2\lambda}\cdot \frac{\Gamma}{\lambda T^{3/2}} + \frac{\Gamma^2}{\lambda^2 T^3} < \frac{3\Gamma^2}{\lambda^2} \]
as required.

Now consider $L_{t+1}$ for $t \geq 2$.  We have
\begin{align*}
\tilde{L}_{t+1} &= \left(1 - \frac{2}{t}\right)L_t + \frac{\Gamma^2}{\lambda^2 t^2} \\
&\leq \left(1 - \frac{2}{t}\right)\frac{7\Gamma^2}{\lambda^2 t} + \frac{\Gamma^2}{\lambda^2 t^2} \\
%&\leq \frac{\Gamma^2}{\lambda^2 (t+1)}(6 - \frac{11}{t})(1 + \frac{1}{t})
&\leq \frac{\Gamma^2}{\lambda^2 (t+1)}\left(7 - \frac{12}{t} - \frac{13}{t^2}\right)
\end{align*}
and hence 
\begin{align*}
L_{t+1} &\leq \tilde{L}_{t+1} + 2\mu\sqrt{\tilde{L}_{t+1}} + \mu^2\\
&\leq \frac{\Gamma^2}{\lambda^2 (t+1)}\left(7 - \frac{12}{t} - \frac{13}{t^2}\right) + 2\cdot \frac{\Gamma\sqrt{7}}{\lambda \sqrt{(t+1)}} \frac{\Gamma}{\lambda T^{3/2}} + \frac{\Gamma^2}{\lambda^2 T^3}\\
&\leq \frac{\Gamma^2}{\lambda^2(t+1)}\left[ 7 - \frac{12-2\sqrt{7}}{t} - \frac{12}{t^2} \right]
< \frac{7\Gamma^2}{\lambda^2(t+1)}
\end{align*}
as required, completing the proof of the lemma.
\end{proof}

\begin{lemmanum}[\ref{lem:mean.errors} (Based on Corollary 24 of~\cite{becchetti2022multidimensional})]

    Given $\epsilon > 0$, suppose $m$ vectors $x_i \in \mathbb{R}^d$ are drawn iid from a $d$-dimensional Gaussian $N(\theta, \Sigma)$ where the eigenvalues of $\Sigma$ lie in $[(\sigma/a)^2, \sigma^2]$ for some $a \geq 1$.  Then there exists a constant $C > 0$ such that if
    \[ m \geq C a d^2 \log^2(d\sigma/a\epsilon)  \]
    then, for any $z \in [\theta-(\sigma/a), \theta+(\sigma/a)]^d$, with probability at least $1 - \epsilon/2$, there exists a subset $S$ of the vectors with average $\textsc{avg}(S)$ such that $||z - \textsc{avg}(S)||_2 < 2\epsilon$.
\end{lemmanum}

\begin{proof}
    We begin with a definition.   We say that a distribution $F_1$ contains another distribution $F_2$ with probability $p$ if the density function of $F_1$ is pointwise larger than the density function of $F_2$ times $p$ over the support of both distributions.
    
    With this definition in hand, we can now restate Corollary 24 of~\cite{becchetti2022multidimensional}:\footnote{As suggested when establishing their Theorem 2, in the restatement of this corollary we set $\alpha = 1/6\sqrt{d}$ and choose $\delta$ so that the error rate is $\epsilon$.}

    \textbf{Corollary:} Let $\sigma > 0$, $\epsilon \in (0,\sigma)$, and let $p \in (0,1]$ be a constant.  Given $d,m \in \mathbb{N}$ let $Y_1, \dotsc, Y_m$ be independent $d$-dimensional random vectors containing $d$-dimensional normal random vectors drawn from $\mathcal{N}(v, \sigma^2 \cdot I_d)$ with probability $p$ where $v$ is any vector in $\mathbb{R}^d$.  Then there exists a universal constant $C > 0$ such that if
    \[ m \geq 2C\frac{d^2}{p}\left(\log\frac{\sigma}{\epsilon} + \log d\right)^2,\]
    then with probability at least $1-\epsilon$, for any $z \in [v -\sigma, v + \sigma]^d$, there exists a subset $S \subseteq [n]$ for which
    \[ ||z - \textsc{avg}(S)||_\infty \leq 2\epsilon. \]

    To prove Lemma~\ref{lem:mean.errors}, we must show that our $d$-dimensional Gaussian distribution $G$ satisfies the conditions of the Corollary.  Note then that for any $\sigma_1 > \sigma_2$, a single-dimensional Gaussian of variance $\sigma_1^2$ contains a Gaussian of variance $\sigma_2^2$ with probability $\sigma_2 / \sigma_1$.  Applying this to each eigenvector for $G$ implies that $G$ contains a $d$-dimensional Gaussian with covariance $(\sigma/a)^2 I_d$ with probability $1/a$.  
    
    Applying Corollary 24 of \cite{becchetti2022multidimensional} with $p = 1/a$ and taking the change of variables $\sigma \leftarrow \sigma/a$ yields that under the desired bound on $m$, there exists a subset $S$ of the vectors with average $\textsc{avg}(S)$ such that $||z-\textsc{avg}(S)||_\infty < 2\epsilon$ and hence $||z-\textsc{avg}(S)||_2 < 2\epsilon\sqrt{d}$.  Taking a change of variables $\epsilon \leftarrow \epsilon/\sqrt{d}$ then yields the desired result.
\end{proof}

\begin{propositionnum}[\ref{prop:mean.simple}]
    Fix $\epsilon > 0$ and suppose $F = \mathcal{N}(\theta, \Sigma)$ where all eigenvalues of $\Sigma$ lie in $[(\sigma/a), \sigma]$ where $a \geq 1$.  Then if we run Algorithm 1 for $T$ rounds and $m$ memory with $T > 12d\sigma^2/m\epsilon$ and $m \geq C d^2 \log^2(d\sigma/a\epsilon) )$, then the expected squared loss after round $T$ is at most $\epsilon$.
\end{propositionnum}
\begin{proof}
First note that the trivial algorithm that chooses $S_t = M_t$ in each round and returns $\textsc{avg}(S_T)$ achieves expected squared error equal to the variance of an average of $m$ draws from distribution $G$, which is at most $d\sigma^2/m$.  It is without loss of generality to assume that $\epsilon < d\sigma^2/m$.

Let $H(s) = \frac{1}{2}||s-\theta||^2$, and note that the gradient of $H$ at $s$ is $(s - \theta)$.  In particular, since $E[y_t] = \theta$ in each round $t$ (for any choice of $b$), this means that $E[(s-y_t)]$ is equal to the gradient of $H$ at $s$.  Thus, as long as $E[||s-\theta||^2] < d\sigma^2/m$, we have $E[||s-y_t||^2] \leq E[(||s-\theta||+||y_t-\theta||)^2] \leq d\sigma^2/m + \textsc{var}(y_t) \leq 2d\sigma^2/m$.

We conclude that if the adversarial error bounds of Lemma~\ref{lem:mean.sgd} are satisfied, our expected squared loss will be at most $C_1 \sigma^2d/Tm$ for a fixed constant $C_1$.  We can therefore take $T \geq C_1 d\sigma^2/(m\epsilon)$ to guarantee expected squared loss $\epsilon$.

It remains to establish the adversarial error bounds.  Recall that set $N_t$ has size $m/2$ in each round $t$.  By Lemma~\ref{lem:mean.errors}, there exists a constant $C$ such that as long as $|N_t| \geq Cad^2\log^2(d\sigma/a\epsilon')$, we will have $||z-\textsc{avg}(S)|| < \epsilon'$ with probability $(1-\epsilon')$.  With the remaining probability $\epsilon'$ we have $E[||z-\textsc{avg}(S)||^2] < \sigma^2/m$ (by taking $S$ to be the set of all vectors), so in particular we have $E[||z-\textsc{avg}(S)||^2] < (\epsilon')^2 + \epsilon' \sigma^2/m$.  Since we require $\mu^2 \leq G^2/T^3 \leq c \epsilon^3 m^2 / d^2 \sigma^4 < c \epsilon$ for a constant $c$ (where we used the assumption that $\epsilon < d\sigma^2/m$) we can choose $\epsilon' = c \epsilon$ for some appropriate constant $c$, yielding the required error bound as long as $m \geq C_2 d^2 \log^2(d\sigma/a\epsilon)$ for a constant $C_2$ depending on the parameters of the input distribution.
\end{proof}

\begin{theoremnum}[\ref{thm:mean}]
    Fix any $\epsilon > 0$ and dimension $d \geq 1$.  There exists a subsampling algorithm $A$ with recency such that if $T > C_1 d/\epsilon$ and $m \geq C_2 d \log(d/\epsilon)$, where $C_1$ and $C_2$ are constants that depend on $\sigma, \gamma, \alpha)$ (parameters of the data distribution) then the expected squared loss at most $\epsilon$.  The update in each round, as well as generating the final output, can each be performed in time $\textsc{Poly}(d, 1/\epsilon)$.
\end{theoremnum}

\begin{proof}
We begin by analyzing the case $d=1$.  In this case $F$ is a distribution over real numbers such that $\textsc{var}(F) \leq \sigma^2$ and $F$ has density at least $\alpha$ on the interval $[\theta - \gamma, \theta + \gamma]$.

Recall the following result from \cite{Lueker:1998}:

\begin{theorem}[Theorem 2.4 of \cite{Lueker:1998}]
\label{thm.subset.sum}
Suppose that $x_1, \dotsc, x_m$ are drawn independently from $U[-1,1]$.  Then if $m > C_2 \log(1/\epsilon)$ then with probability at least $1 - e^{-C_1 m}$ there is a subset $S$ of the items with $|z-\sum_{x \in S}x| < \epsilon$, for any $z \in [-1/2, 1/2]$.
\end{theorem}

Theorem~\ref{thm:subset.sum} suggests the following improved analysis for Algorithm 1 in the single-dimensional case $d=1$.  Recall that we assume distribution $F$ has density at least $\alpha$ on the interval $[\theta - \gamma, \theta + \gamma]$.  Choose any $z \in [\theta - \gamma/2, \theta + \gamma/2]$.  Then $F$ has density at least $\alpha$ on the interval $[z-\gamma/2, z+\gamma/2]$.  We can therefore think of distribution $F$ as drawing a value uniformly from $[z-\gamma/2, z+\gamma/2]$ with probability $\alpha\gamma$, and drawing a value from some arbitrary residual distribution otherwise.  As long as $m > 2 C_2 (\alpha\gamma)^{-1} \log(1/\epsilon)$, Chernoff bounds imply that at least $C_2 \log(1/\epsilon)$ items will be drawn from this uniform distribution with probability $1-e^{-C_3 m}$ for some constant $C_3$.  For each of those items $x_i$, note that $(x_i - z)$ is uniformly distributed on $[-\gamma/2, \gamma/2]$.  Then by Theorem~\ref{thm:subset.sum}, with probability at least $1 - e^{-C_1 C_2 \log(1/\epsilon)}$ there will be a subset $S$ of those items such that $|\sum_{x \in S}(x-z)| < \gamma\epsilon/2$.  This means that $|\textsc{avg}(S)-z| < \gamma\epsilon/2 < \epsilon$ as well.

We now proceed similarly to the proof of Proposition~\ref{prop:mean.simple}, employing the bounds above in the place of Lemma~\ref{lem:mean.errors}.  First, as in Proposition~\ref{prop:mean.simple} it is without loss to assume that $\epsilon < \sigma^2/m$, as otherwise the result is satisfied by a trivial algorithm that sets $S_t = M_t$.  Next, setting $H(s) = \frac{1}{2}|s-\theta|^2$, the gradient of $H$ at $s$ is $(s - \theta)$.  In particular, since $E[y_t] = \theta$ in each round $t$, this means that $E[(s-y_t)]$ is equal to the gradient of $H$ at $s$.  Thus, as long as $E[|s-\theta|^2] < \sigma^2/m$, we have $E[|s-y_t|^2] \leq E[(|s-\theta|+|y_t-\theta|)^2] \leq \sigma^2/m + \textsc{var}(y_t) \leq 2\sigma^2/m$, where the last inequality follows because $y_t$ is an average over $m/2$ draws from distribution $F$.

We conclude that if $E[|z - \textsc{avg}(S)|^2] < \Gamma^2/T^3$ (where $\Gamma^2 = 2\sigma^2/m$), then the adversarial error bounds of Lemma~\ref{lem:mean.sgd} are satisfied, so by Lemma~\ref{lem:mean.sgd} our expected squared loss at the termination of the algorithm will be at most $C_1 \sigma^2d/Tm$ for a fixed constant $C_1$.  We can therefore take $T \geq C_1 d\sigma^2/(m\epsilon)$ to guarantee expected squared loss $\epsilon$ as required.

To establish the adversarial error bounds required by Lemma~\ref{lem:mean.sgd} (i.e., that $E[|z - \textsc{avg}(S)|^2] < \Gamma^2/T^3$), we use the fact (derived above) that as long as $m > 2 C_2 (\alpha\gamma)^{-1} \log(1/\epsilon')$, with probability at least $1 - O(\epsilon')$ we will have $|z - \textsc{avg}(S)| < \epsilon'$.  With the remaining probability $\epsilon'$ we will have $E[|z - \textsc{avg}(S)|^2] < \sigma^2/m$ (taking $S$ to be the set of all vectors).  So we have the unconditional bound $E[|z - \textsc{avg}(S)|^2] < (\epsilon')^2 + \epsilon' \sigma^2/m$.  Since we require $E[|z - \textsc{avg}(S)|^2] \leq \Gamma^2/T^3 \leq c \epsilon^3 m^2 / \sigma^4 < c \epsilon$ for constant $c$ (using the assumption that $\epsilon < \sigma^2/m$), we can choose $\epsilon' = c \epsilon$ for some appropriate constant $c$, yielding the required error bound for the case $d=1$.

Moreover, note updates under this algorithm can be computed in time $2^m = \textsc{Poly}(1/\epsilon, 2^{(\alpha\gamma)^{-1}})$, as this is the time required to check all $2^m$ subsets of items.  For $\alpha, \gamma = O(1)$ this is $\textsc{Poly}(1/\epsilon)$.  This completes the proof of Theorem~\ref{thm:mean} for the case $d=1$.

For the more general $d$-dimensional case we will not use Algorithm 1 directly, but rather we will reduce to the $1$-dimensional setting as follows.  For each dimension $i$ from $1$ to $d$, simulate (in parallel) a separate single-dimensional instance of Algorithm 1 to estimate $\theta_i$, each with its own separate segment of memory.  This inflates the memory requirement by a factor of $d$, so we require $m > C_2 d (\alpha\gamma)^{-1} \log(1/\epsilon')$ to achieve squared error at most $\epsilon'$ in each dimension.  The resulting (total) squared error will be at most $d\epsilon'$, so to achieve error $\epsilon$ we will perform a change of variables $\epsilon' = \epsilon/d$, resulting in memory requirement $m > C_2 d (\alpha\gamma)^{-1} \log(d/\epsilon)$ and time requirement $T = \Theta(d \sigma^2/m\epsilon)$.  Simulating $d$ instances of the algorithm in parallel also inflates the computation time by a factor of $d$, so the running time per update is polynomial in $1/\epsilon$ and $d$. This setting of parameters therefore satisfies the requirements of Theorem~\ref{thm:mean}.
\end{proof}

\begin{theoremnum}[\ref{thm:lowerbound}]
Suppose data items are drawn from a standard normal distribution $\mathcal{N}(\theta, I_d)$ (i.e., with each dimension drawn from an independent Gaussian with unit variance).  Choose any $\epsilon > 0$.  If $m < \frac{d\log(1/\epsilon)}{\log d + \log\log(1/\epsilon)}$ then for any algorithm that satisfies recency and any $T$, with probability at least $2/3$ the squared error will be strictly greater than $\epsilon$.
\end{theoremnum}

\begin{proof}
We will consider algorithms in the batch model.
Fix an arbitrary estimator function $f$ that will map the final state $S_T$ to an estimate of $\theta$.  Our proof strategy will be to assume that $\theta$ is known and revealed after the estimator $f$ is chosen, and the algorithm is allowed to use that knowledge of $\theta$ when deciding which subset of data points to retain each round.  Under this assumption the behavior of the algorithm before round $T$ is irrelevant, and the only question is whether there exists any subset $S \subseteq M_T$ of the final set of data items such that $||f(S) - \theta||_2^2 < \epsilon$.

In fact, let's give ourselves additional power and let our estimator $f$ be parameterized by the set of indices of the chosen set $S$, in addition to its realization.  Let $f_X(S)$ denote the estimator for the subset of indices $X \subseteq [m]$, as a function of the realization $S$ of items corresponding to those indices in $M_T$.  Our approach will be to bound the probability that $||f_X(S) - \theta||_2^2 < \epsilon$ for a fixed choice of $X$, then take a union bound over all $X$.

We first claim that for all $X$, the estimator $f_X$ that maximizes the probability of being within $\epsilon$ squared distance to $\theta$ is the mean estimator.  That is, for all $X$ and all $f_X$, if we write $\overline{S} = \frac{1}{|X|}\sum_{i \in X}x_i$ to denote the mean of the elements of $X$ under realization $S$, we have 
\begin{align}
\label{eq:mean.estimator}
\min_\theta \Pr_S[||f_X(S) - \theta||_2^2 < \epsilon] \leq \min_\theta \Pr_S\left[||\overline{S} - \theta||_2^2 < \epsilon\right].
\end{align}

That is, for any $X$ and any estimator $f_X$, the worst-case (over $\theta$) probability of being within $\epsilon$ of $\theta$ is no greater than for the mean estimator.  To see why, first note that the LHS is at most the expected probability over $\theta$ drawn from a diffuse prior.  So if we denote the (improper) diffuse prior by $\Gamma$, we have
\begin{align*}
    \min_\theta \Pr_S[||f_X(S) - \theta||_2^2 < \epsilon] 
    &\leq \Pr_{\theta \sim \Gamma} \Pr_{S|\theta}[||f_X(S) - \theta||_2^2 < \epsilon]\\
    &= \Pr_{S} \Pr{\theta|S}[||f_X(S) - \theta||_2^2 < \epsilon].
\end{align*}
But for any realization $S$ of the items with indices in $X$, the posterior distribution of $\theta$ given $S$ is precisely a Gaussian with mean equal to the empirical mean $\overline{S}$ of $S$.  Since the $\epsilon$ ball most likely to contain a draw of $\theta$ from this posterior is the one centered at $\overline{S}$, we conclude
\begin{align*}
    \min_\theta \Pr_S[||f_X(S) - \theta||_2^2 < \epsilon] 
    & \leq \Pr_{S} \Pr{\theta|S}[||f_X(S) - \theta||_2^2 < \epsilon]\\
    & \leq \Pr_{S} \Pr{\theta|S}[||\overline{S} - \theta||_2^2 < \epsilon]\\
    & = \Pr_{\theta} \Pr{S|\theta}[||\overline{S} - \theta||_2^2 < \epsilon]\\
    & = \min_\theta \Pr{S|\theta}[||\overline{S} - \theta||_2^2 < \epsilon]
\end{align*}
as claimed, where the last equality follows because both the distribution of $S$ as well as the estimator $\overline{S}$ are translation invariant with respect to $\theta$, so indeed the probability is identical for all realizations of $\theta$.

Given \eqref{eq:mean.estimator}, it suffices to bound the probability of the event $||f_X(S) - \theta||_2^2 < \epsilon$ for the empirical mean estimator.  If we assume our noise distribution on each dimension is a Gaussian with variance $1$, then the posterior of $\theta$ has density at most $O(\sqrt{m})$ at the empirical mean of $S$.  So the probability that $\theta$ falls in a ball of radius $\sqrt{\epsilon}$ is at most $c\epsilon^{d/2}\sqrt{m}$ for some constant $c$.

There are $2^m$ possible choices of $X$.  So no matter what estimators are selected, a union bound implies that the probability that there exists $X \subseteq [m]$ such that $||\theta-f_X(S)||_2^2 < \epsilon$ is at most $c\epsilon^{d/2}2^m\sqrt{m}$.

This means that this probability will be at most $1/3$ unless $c\epsilon^{d/2}2^m\sqrt{m} > 1/3$, which is equivalent to $m\log m > d \log(1/\epsilon) + c'$ for some constant $c'$.  So in particular this requires that $m > \frac{d \log(1/\epsilon)}{\log d + \log\log(1/\epsilon)}$, for sufficiently small $\epsilon$.
\end{proof}

\section{Omitted Proofs from Section 4}
\label{app:proofs.regression}

\begin{lemmanum}[\ref{lem:matrix.concentration}]
There exist constants $C$, $\alpha$, and $\gamma$ depending on $\sigma^2$ (the variance of the noise distribution) and the distribution $G$ over data items such that the following holds. Let $(x_1, y_1), \dotsc, (x_k,y_k)$ be independent draws from our data generating process, where $k > C d \log d$.  Write $\theta' = (X^T X)^{-1}X^TY$, where $X$ and $Y$ are the matrices of predictors and outcomes. Then for any $i \in [d]$, the distribution over $\theta'_i$ has density at least $\alpha$ in the range $[\theta_i - \gamma, \theta_i + \gamma]$.
\end{lemmanum}

\begin{proof}
Let $A = E[x^T x]$ where $x$ is drawn from our distribution.  By assumption $A$ is invertible as well as positive semidefinite, and has eigenvalues bounded in $[\lambda_1, \lambda_2]$.  Moreover, note that $X^T X = \sum_i x_i^T x_i$, so $E[X^T X] = kA$.  This also means that we can view $X^T X$ as a sum of $k$ matrices $x_i^T x_i$, each of which is positive semidefinite and has maximum eigenvalue bounded by $dB^2$ with probability $1$ (where recall $B$ is the upper bound on $x$, and hence $B^2$ is an upper bound on any element of $A$).

By the Matrix Chernoff bound, the average matrix $\frac{1}{k} X^T X$ has eigenvalues bounded in $[\lambda_1/2, 2\lambda_2]$ with probability at least $1 - 2 d e^{-c k \lambda_1 / d B^2}$ for some constant $c$.  Thus, as long as $k > C d \log d$ for some constant $C$, this event occurs with at least constant probability.

Recalling that $Y = X \theta + \epsilon$, we have that $\theta' = \theta + (X^T X)^{-1} X^T \epsilon$.  Given that the event above occurs, we have that $(X^T X)$ is invertible, and $(X^T X)^{-1}$ has eigenvalues lying in a range $[\mu_1/k, \mu_2/k]$ where $\mu_1$ and $\mu_2$ depend only on $\lambda_1$ and $\lambda_2$.  Moreover, for any fixed  realization of $X$ and considering the distribution over the realization of the noise term $\epsilon$, the random variable $X^T \epsilon$ is a sum of $k$ mean-zero Gaussians each with variance at most $B \sigma^2$, for a total variance of at most $B \sigma^2 k$.  

We conclude that with constant probability, $X$ will be such that the random vector (with respect to noise $\epsilon$) $(X^T X)^{-1}X^T \epsilon$ will be distributed as a mean-zero Gaussian with covariance eigenvalues bounded in some fixed range $[\mu_1\sigma, \mu_2\sigma]$.  But from this we conclude that coordinate $i$ of $(X^T X)^{-1}X^T \epsilon$ will have positive density at least $\alpha$ in some range $[-\gamma, \gamma]$, where $\gamma$ and $\alpha$ depend on $\lambda_1, \lambda_2$, and $\sigma$.  Thus $\theta'_i$ will have density at least $\alpha$ on $[\theta_i - \gamma, \theta_i + \gamma]$ as required.
\end{proof}

\begin{theoremnum}[\ref{thm:regression}]
    Fix any $\epsilon > 0$ and consider the linear regression task.  There exists a subsampling algorithm $A$ with recency such that if $T > C_1 d/\epsilon$ and $m \geq C_2 d^2 \log(d) \log(d/\epsilon)$, where $C_1$ and $C_2$ are constants that depend on $\sigma, B, \lambda_0, \lambda_1$ (parameters of the data distribution) then the expected squared loss at most $\epsilon$.  The update in each round, as well as generating the final output, can each be performed in time $\textsc{Poly}(d, 1/\epsilon)$.
\end{theoremnum}

\begin{proof}[Proof of Theorem~\ref{thm:regression}]
We will invoke Lemma~\ref{lem:mean.sgd}, treating each $s_t$ as following a round of stochastic gradient descent (moving from $s_{t-1}$ to $z_t$) followed by adversarial noise (moving from $z_t$ to $s_t$).  We first bound the variance of the gradient in each step, then bound the extent of the adversarial noise.

Given a set of $m$ data items $X$ and $Y$, let $H(s) = \frac{1}{2m}||Xs - Y||^2$ denote our mean squared prediction error with respect to an arbitrary estimate $s \in \reals^d$ for $\theta$.  The gradient of this objective with respect to $s$ is $\frac{1}{m}X^T(Xs - Y)$.  Its variance is $C_1 d / m$ where $d$ is the dimension, $m$ is the number of data samples, and $C$ is a constant that depends on the error variance $\sigma$ and the distribution from which each $x_t$ is drawn.  Furthermore, the expectation of this gradient (with respect to $X$ and $Y$) is $E_{x \sim G}[x^T x](s-\theta)$, which is the gradient of $\frac{1}{2}||(E_{x \sim G}[x^T x])^T(s - \theta)||^2$.  We can therefore interpret $z_t$ as a stochastic gradient update step with variance $C_2d/m$ for some constant $C_2$, for objective value $\frac{1}{2}||(E_{x \sim G}[x^T x])^T(s - \theta)||^2 \leq \frac{1}{2}||(E_{x \sim G}[x^T x])||^2||(s - \theta)||^2 = C_3 ||(s-\theta)||^2$ where $C_3$ is a constant depending on the primitives of the model.

We next define $\zeta_t = (s_t - z_t)$ and bound $E[||\zeta_t||^2]$. Invoking Lemma~\ref{lem:mean.sgd} requires that $E[||\zeta_t||^2] < \Gamma^2/T^3$, which will be $O(\epsilon / d^2)$ for $T = \Theta(d/m\epsilon)$ (as in our analysis of mean estimation). Lemma~\ref{lem:matrix.concentration} and Theorem 6 (the analysis of single-dimensional random subset sum, restating Theorem 2.4 of \cite{Lueker:1998}) together imply that as long as $r \geq C_4 (\alpha \gamma)^{-1} \log(d/\epsilon')$ (recalling that $r$ is the number of size-$k$ groups of data items available for encoding each dimension of $z_t$), our expected squared error in the encoding of $[z_t]_i$ will be at most $\epsilon' / d$ for each dimension $i$, and thus at most $\epsilon'$ all together.  As each batch has size $C_5 d \log d$ for some constant $C_5$, and the set $N_t$ has size $m/2$ and must contain $d r$ total batches, our memory requirement is
\[ m \geq 2 d C_4 (\alpha \gamma)^{-1} \log(d/\epsilon') (C_5 d \log d) = C_6 d^2 \log(d) \log(d/\epsilon') \]
where $C_6$ is a constant depending on the distribution over input items.  Taking $\epsilon' = O(\epsilon / d^2)$ then satisfies the required bound on $E[||\zeta_t||^2]$.

Thus, as long as the total number of data points used for simulating stochastic gradient descent, $mT/2$, is at least $C_7d/\epsilon$ where $C_7$ depends on the input distribution, we conclude that the expected error at the termination of the final round is at most $\epsilon$ as required.
\end{proof}

\section{Lack of Full Differential Privacy: An Example}\label{app:df}

In this section, we provide a worked example of how differential privacy is violated through the maintenance of a state $S_t$, even though old data has been deleted. Let $m = 3$ and suppose that we initialize $M_1 = 0$; at $t = 2$, we receive a batch of data $M_2 = (0, 10, 10)$. There are two ways (up to isomorphism) that Algorithm 1 can partition the data into $(R_t, N_t)$ such that $|R_t| = 1$ and $|N_t|=2$:
\begin{enumerate}
    \item $R_t = \{0\}$ and $N_t = \{10, 10\}$: The mean of $R_t$ is 0, so the target mean is $(0+0)/2 = 0$ and $S_t = \{10\}$.
    \item $R_t = \{10\}$ and $N_t = \{0, 10\}$: The mean  of $R_t$ is 10, so the target mean is $(10+0)/2 = 5$ and so $S_t = \{0, 10\}$.
\end{enumerate}
The image set from $M_2$ is $\{\{10\}, \{0, 10\}\}$ corresponding to the set of means $\{10, 5\}$. Now instead suppose we receive a batch of data $M_2' = (0, 0, 10)$ and notice that this batch differs from $M_2$ only in its second data point. Once again, there are two ways (up to isomorphism) that Algorithm 1 can partition the data into $(R_t, N_t)$ where $|R_t| = |N_t| = 2$:
\begin{enumerate}
    \item $R_t = \{0\}$ and $N_t = \{0, 10\}$: The mean of $R_t$ is 0, so the target mean is $(0+0)/2$ and so $S_t = \{0\}$.
    \item $R_t = \{10\}$ and $N_t = \{0, 0\}$: The mean  of $R_t$ is 10, so the target mean is $(10+0)/2 = 5$ and so $S_t = \{0\}$.
\end{enumerate}
The image set from $M_2'$ is $\{\{0\}\}$ corresponding to the set of means $\{0\}$. Notice there is no intersection between the image sets from $M_2$ and $M_2'$. Moreover, observe that in none of the image sets (from either $M_2$ or $M_2'$) does the second data point need to be retained. Thus, the second data point can be deleted but the state retained $S_2$ perfectly reveals whether the batch was $M_2$ or $M_2'$, which differ only in their second data point.

\end{document}